\documentclass[10pt,twocolumn]{article}

\usepackage[margin=0.75in]{geometry}
\usepackage[utf8x]{inputenc}
\usepackage[T1]{fontenc}
\usepackage{times}
\usepackage{cite}
\usepackage{graphics} % for pdf, bitmapped graphics files
\usepackage{epsfig} % for postscript graphics files
\usepackage{times} % assumes new font selection scheme installed
\usepackage{amsmath} % assumes amsmath package installed
\usepackage{amssymb,amsfonts,amscd,dsfont,nicefrac,amsmath,amsthm} 
\usepackage{subcaption}
\usepackage{enumerate,url}

\usepackage{titlesec}

\titleformat{\section}
  {\normalfont\large\bfseries}{\thesection.}{0.5em}{}
  
\titleformat{\subsection}
  {\normalfont\normalsize\itshape}{\thesubsection.}{1em}{}
  
\usepackage[ruled,linesnumbered,lined,vlined]{algorithm2e}

\usepackage{xcolor}
\usepackage{balance}

\renewcommand{\Re}{\mathbb{R}}
\newcommand{\Nint}{\mathcal N_{\text{int}}}
\newcommand{\Nlf}{\mathcal N_{\text{leaf}}}
\newcommand{\T}{\mathcal T}
\newcommand{\N}{\mathcal N}
\newcommand{\chd}{\mathcal C}
\newcommand{\W}{\mathcal W}
\newcommand{\Q}{\mathcal O}

\newcommand{\Remp}{\mathbb{R}^{|\calK_{Z}|}_+}
\newcommand{\Renp}{\mathbb{R}^{|\calK_{Y}|}_+}

\newcommand{\JS}{\mathrm{JS}_{\Pi}}
\newcommand{\DKL}{\mathrm D_{\mathrm{KL}}}

\newcommand{\algoref}[1]{Algorithm~\ref{#1}}
\newcommand{\figref}[1]{Fig.~\ref{#1}}
\newcommand{\mfigref}[2]{Figs.~\ref{#1}-\ref{#2}}
\newcommand{\twoFigRef}[2]{Figs.~\ref{#1} and~\ref{#2}}

\newtheorem{theorem}{Theorem}[section]

\newtheorem{proposition}[theorem]{Proposition}

\newcommand{\calK}{{\cal K}}

\newcommand{\bbR}{\mathbb{R}}

\renewcommand\footnotemark{}

\renewcommand{\thesection}{\arabic{section}} 
\renewcommand{\thesubsection}{\Alph{subsection}}

\title{\LARGE \bf Information-theoretic Abstraction of Semantic Octree Models for Integrated Perception and Planning}
\date{}
\author{Daniel T. Larsson$^{1}$~~ Arash Asgharivaskasi$^{2}$~~ Jaein Lim$^{1}$~~ Nikolay Atanasov$^{2}$~~ Panagiotis Tsiotras$^{1}$% <-this % stops a space
\footnote{This research was funded by Office of Naval Research award N00014-18-1-2375 and by the Army Research Laboratory under DCIST CRA W911NF-17-2-0181.}% <-this % stops a space (\emph{Corresponding Author: Daniel Larsson.})
\footnote{$^{1}$The authors are with the D. Guggenheim School of Aerospace Engineering, Georgia Institute of Technology, Atlanta, GA, 30332-0150, USA. {\tt\small $\{$dlarsson3,jaeinlim126,tsiotras$\}$@gatech.edu}}%
\footnote{$^{2}$The authors are with the Department of Electrical and Computer Engineering, University of California San Diego, San Diego, CA 92093, USA. {\tt\small $\{$aasghari,natanasov$\}$@ucsd.edu}}%
}
\begin{document} 

\maketitle
%\thispagestyle{empty}
%\pagestyle{fancy}

%%%%%%%%%%%%%%%%%%%%%%%%%%%%%%%%%%%%%%%%%%%%%%%%%%%%%%%%%%%%%%%%%%%%%%%%%%%%%%%%
\begin{abstract}
In this paper, we develop an approach that enables autonomous robots to build and compress semantic environment representations from point-cloud data.
Our approach builds a three-dimensional, semantic tree representation of the environment from sensor data which is then compressed by a novel information-theoretic tree-pruning approach.
The proposed approach is probabilistic and incorporates the uncertainty in semantic classification inherent in real-world environments.
Moreover, our approach allows robots to prioritize individual semantic classes when generating the compressed trees, so as to  design multi-resolution representations that retain the relevant semantic information while simultaneously discarding unwanted semantic categories.
We demonstrate the approach by compressing semantic octree models of a large outdoor, semantically rich, real-world environment.
In addition, we show how the octree abstractions can be used to create semantically-informed graphs for motion planning, and provide a comparison of our approach with uninformed graph construction methods such as Halton sequences.
\end{abstract}
%%%%%%%%%%%%%%%%%%%%%%%%%%%%%%%%%%%%%%%%%%%%%%%%%%%%%%%%%%%%%%%%%%%%%%%%%%%%%%%%

%%%%%%%%%%%%%%%%%%%%%%%%%%%%%%%%%%%%%%%%%%%%%%%%%%%%%%%%%%%%%%%%%%%%%%%%%%%%%%%%%%%%%%
\section{Introduction}

Dense volumetric environment representations such as occupancy grid maps~\cite{occ_mapping_1, occ_mapping_2}, multi-resolution hierarchical models~\cite{quadtree, octomap}, and signed distance fields (SDF)~\cite{isdf, sdf}, provide valuable information to both human and autonomous robots, as evidenced by their utility in search and rescue~\cite{search-and-rescue}, safe navigation~\cite{navigation-using-dense-map}, and terrain modeling~\cite{traversability-analysis}.
Moreover, the inclusion of semantic information, such as in metric-semantic SLAM methods of~\cite{zobeidi-GP,semantic-octomap,kimera}, allows robots to build more sophisticated world models by affording autonomous systems the ability to not only discern occupied from free space, but to also distinguish between the types of objects in their surroundings.
As evidence of their usefulness, recent frameworks have leveraged the power of Bayesian statistics to develop algorithms that build semantic environment representations that encode categorical (semantic) information using probabilistic methods which naturally capture the uncertainty robots hold regarding their world~\cite{zobeidi-GP,semantic-octomap}.
These models supply autonomous robots an abundance of information, enabling them to intelligently reason about their surroundings with details such as the location, geometry and size of obstacles, or the presence of humans, cars, and other semantic information.
While constructing environment models is an important step for intelligent autonomy, sensors often provide an overabundance of information for specific tasks.
It is therefore of interest to not only \emph{build} environment models but also \emph{compress} them to form abstracted representations, allowing robots to focus their (possibly scarce) resources on the relevant aspects of the operating domain.
The use of abstractions in the form of multi-resolution environment model compressions have seen widespread deployment in the autonomous systems community.
Examples include~\cite{Hauer2015,Cowlagi2010,Larsson2021:AbsPlan,Kambhampati1986,lim2020mams,Boutilier1994}, where abstractions are leveraged in order to alleviate the computational cost of planning and decision making in both single and multi-robot applications.
Abstractions have also been utilized to reduce the memory required to store environment representations~\cite{Kraetzschmar2004,Einhorn2011} and to alleviate the computational complexity of evaluating cost functions in active-sensing applications~\cite{Nelson2018}.
However, while identifying the relevant aspects of a problem to generate task-relevant abstractions has long been considered vital to intelligent reasoning~\cite{Zucker2003,brooks1991intelligence,Sacerdoti1974,Giunchiglia1992,holte2003abstraction,Ponsen2010}, the means by which they are generated has traditionally been heavily reliant on user-provided rules.
To this end, a number of studies have considered the generation of task-relevant abstractions for control and decision-making that model (relevant) information via the statistics of the process.
Examples of such works include~\cite{Genewein2015,Tishby2010,pacelli2019task}, where ideas from information theory, specifically rate-distortion~\cite{gallager1968information} and the information bottleneck method~\cite{Tishby1999:IB}, are employed to develop approaches that identify and preserve task-relevant information by modeling the relevant information as a random variable that is correlated with the source (i.e., the original representation).  
While the above studies consider frameworks which form abstractions that preserve relevant information, they do not result with representations of any particular structure (e.g., quadtrees, octrees).
To address this issue, the work of~\cite{larsson2020q,larsson2021HCinformation} uncovered connections between hierarchical tree structures and signal encoders to formulate an information-theoretic compression problem that allows optimal task-relevant tree abstractions to be obtained as a solution to an optimization problem.
Moreover, extensions of the tree abstraction problem from~\cite{larsson2020q} that considers generating tree abstractions in the presence of both relevant and irrelevant information sources has recently appeared in the literature~\cite{larsson2022generalized}.
Importantly, the frameworks developed in~\cite{larsson2020q,larsson2021HCinformation,larsson2022generalized} require minimal input from system designers and enable robots to generate compressed tree representations of their world that are driven by task-specific information.

The goal of this paper is to bridge the gap between map building and abstraction construction by developing a framework that performs both tasks simultaneously.
The proposed approach employs an information-theoretic tree compression method to find provably optimal tree abstractions of large environments that both retain information regarding task-relevant semantic classes and remove those that are considered task-irrelevant.
We demonstrate our approach in a real-world outdoor environment, and show how the framework can be employed to generate semantically-informed (colored) graphs to reduce the computational effort required for motion planning.
%

%%%%%%%%%%%%%%%%%%%%%%%%%%%%%%%%%%%%%%%%%%%%%%%%%%%%%%%%%%%%%%%%%%%%%%%%%%%%%%%%
\section{Problem Statement}
\label{sec:problem_statement}

Let \((\Omega,\mathcal F,\mathbb P)\) be a probability space.
The \emph{Shannon entropy}~\cite[p.~14]{Cover2006:ElementsITBook} of a discrete random variable \(X: \Omega \rightarrow \bbR\) with probability mass function \(p(x) = \mathbb P\{\omega\in\Omega:X(\omega) = x\}\) is denoted \(H(X)\).\footnote{if the distribution \(p(x)\) is understood from context we may write \(H(p)\) in place of \(H(X)\).}
Provided two distributions \(p(x)\) and \(\bar p(x)\) over the same set of outcomes, the \emph{Kullback-Leibler divergence}~\cite[p.~19]{Cover2006:ElementsITBook} is \(\DKL(p(x),\bar p(x)) = \sum_x p(x) \log [p(x)/\bar p(x)] \).
Given a collection of distributions \(p_1(x),\ldots,p_l(x)\) over the same set of outcomes, the \emph{Jensen-Shannon divergence}~\cite{Lin1991} with respect to the weights \(\Pi\in\Re_+^n\) is given by \(\JS(p_1(x),\ldots,p_l(x)) = \sum_{i=1}^{l} \Pi_i \DKL(p_i(x),\bar p(x))\), where \(\bar p(x) = \sum_{i=1}^{l} \Pi_i p_i(x)\), \(0 \leq \Pi_i \leq 1\) for all \(1\leq i\leq l\) and \(\sum_{i=1}^{l} \Pi_i = 1\).
We assume a grid-world representation of the environment \(\W \subset \Re^3\), where each cell contains semantic information regarding one or more of \(K\) possible semantic classes contained in the set \(\calK = \{0,1,\ldots,K\}\).
In the sequel, we let the semantic class $0\in\calK$ denote free space and each \(k\in\calK\setminus\{0\}\) represent a distinct semantic category (e.g., building, vegetation, road, etc.).
A hierarchical, three-dimensional ($3$-D) multi-resolution octree representation of $\W$ is a tree\footnote{a tree is an a-cyclic connected graph~\cite{Bondy1976}.} $\T$ consisting of a set of nodes \(\N(\T)\) and edges \(\mathcal E(\T)\) that describe the node interconnections, where each non-leaf node in the tree has exactly $8$ children.
We denote the set of children of any node \(n \in \N(\T)\) by \(\chd(n)\), the leaf nodes by \(\Nlf(\T)\), and the interior nodes by \(\Nint(\T)\)~\cite{larsson2022generalized}. 
Lastly, we let \(\T_\W\) denote the finest-resolution octree representation of \(\W\); that is, \(\T_\W\) is the octree whose leaf nodes coincide with the unit cells of \(\W\).  
Our goal is to develop a perception and abstraction approach that allows for semantic octree representations to be built from sensor data while simultaneously optimally compressed in a low-cardinality tree data structure.
For this, we require two components: (i) the source (i.e., the quantity to be compressed) and (ii) any relevant or irrelevant information that is to be retained or removed, respectively.
To this end, the source, relevant and irrelevant information are represented by the random variables \(X:\Omega\to\Nlf(\T_\W)\) with distribution \(p(x)\) (i.e., the finest-resolution cells), \(Y_i: \Omega \to \{0, 1\}\), \(i \in \calK_{Y}\) and \(Z_j:\Omega \to \{0, 1\}\), \(j \in \calK_{Z}\), respectively, where the sets \(\calK_{Y}\) and \(\calK_{Z}\) are subsets of $\calK$ that contain the indices of the relevant and irrelevant semantic classes.
The relationship between source, relevant and irrelevant information is specified by \(p(x, y_{1:|\calK_{Y}|}, z_{1:|\calK_{Z}|})\).
We consider the following problem.
\textit{Semantic Octree Building-Compression:}
Given an octree representation \(\T_\W\) and the joint probability distribution \(p(x, y_{1:|\calK_{Y}|}, z_{1:|\calK_{Z}|})\) from perceptual data, find a compressed multi-resolution octree \(\T\) from \(\T_\W\) by solving the problem:
\begin{equation}\label{eq:GeneralTreeSrchProblem}
    \max_{\T\in\T^\Q} \sum_{i \in \calK_{Y}} \beta_i I_{Y_i}(\T) - \sum_{j \in \calK_{Z}} \gamma_j I_{Z_j}(\T) - \alpha I_X(\T),
\end{equation}
where $\T^\Q$ is the space of all octree representations of $\W$, \(\beta\in \bbR_{+}^{|\calK_{Y}|}\), \(\gamma\in \bbR_{+}^{|\calK_{Z}|}\) and \(\alpha \in \Re_+\) specify the relative importance of relevant information retention, irrelevant information removal, and compression, respectively, and the functions \(I_{Y_i}:\T^\Q \to \Re_+,~I_{Z_j}:\T^\Q \to \Re_+\) and \(I_X:\T^\Q\to\Re_+\) quantify the amount of relevant, irrelevant and compression information contained in the octree (see~\cite[p.~9]{larsson2022generalized}).
Note that \(p(x, y_{1:|\calK_{Y}|}, z_{1:|\calK_{Z}|})\) enters into~\eqref{eq:GeneralTreeSrchProblem} via \(I_{Y_i}(\T)\), \(I_{Z_j}(\T)\) and \(I_X(\T)\).
%

%%%%%%%%%%%%%%%%%%%%%%%%%%%%%%%%%%%%%%%%%%%%%%%%%%%%%%%%%%%%%%%%%%%%%%%%%%%%%%%%
\section{Information-theoretic Abstraction of Semantic Octrees} \label{sec:solnApproach}

We first provide background on tree compression. We will employ the G-tree search algorithm \cite{larsson2022generalized} to solve the compression problem in \eqref{eq:GeneralTreeSrchProblem}.
The goal of the G-tree search algorithm is to find a compressed representation \(N:\Omega \to \Nlf(\T)\), $\T \in \T^\Q$, of the source \(X\) (i.e., the leaf cells of $\T_\W$) in the form of an octree \(\T\) of \(\W\) according to \eqref{eq:GeneralTreeSrchProblem}, where the distribution \(p(n)\) is related to the source according to:
\begin{equation}
    p(n) = \sum_{x\in \Nlf(T_{\W(n)})}p(x),
\end{equation}
and \(\Nlf(\T_{\W(n)})\subseteq \Nlf(\T_\W)\) are the leaf nodes of the subtree of $\T_\W$ rooted node \(n\in\Nlf(\T)\).
Since the octree solution \(\T\in\T^\Q\) to~\eqref{eq:GeneralTreeSrchProblem} is not known a-priori, we may compute the value of \(p(n)\) for all nodes \(n\in\N(\T_\W)\) recursively according to
\begin{equation}\label{eq:parentPx}
    p(n) = \sum_{n' \in \chd(n)} p(n').
\end{equation}

The objective of G-tree search is to determine which nodes of the original octree \(\T_\W\) should be leaf nodes of the compressed representation $\T$.
To accomplish its goal, G-tree search exploits the structure of problem~\eqref{eq:GeneralTreeSrchProblem} to devise a node pruning rule, called the G-function, defined by:
\begin{align}\label{eq:fullGfunction}
&G(n;\beta,\gamma,\alpha) = \nonumber  \\
&~~~~~\max\{\Delta J(n;\beta,\gamma,\alpha) + \sum_{n'\in\chd(n)} G(n';\beta,\gamma,\alpha), ~ 0\},
\end{align}
if \( n\in\Nint(\T_\W)\) and \(p(n) > 0\), and \(G(n;\beta,\gamma,\alpha) = 0\) otherwise.
The function \(\Delta J(n;\beta,\gamma,\alpha)\) is the one-step reward for expanding node \(n\in\Nint(\T_\W)\) and is given by 
\begin{align}
&\Delta J(n;\beta,\gamma,\alpha) = \nonumber \\
&~~~~ \sum_{i \in \calK_Y} \beta_i \Delta I_{Y_i}(n) - \sum_{j \in\calK_Z} \gamma_j \Delta I_{Z_j}(n) - \alpha \Delta I_X(n), \label{eq:oneStepIncrementalCostReward}
\end{align}
where the functions \(\Delta I_{Y_i}(n)\), \(\Delta I_{Z_j}(n)\) and \(\Delta I_X(n)\) quantify the incremental amount of relevant, irrelevant and compression information, respectively, contributed by the node \(n\in \Nint(\T_\W)\).
These functions are, in turn, defined by:
\begin{align}
\Delta I_{Y_i}(n) &= p(n) \JS(p(y_i|n'_1),\ldots,p(y_i|n'_{\lvert \chd(n) \rvert})), \label{eq:deltaIy}\\
\Delta I_{Z_j}(n) &= p(n) \JS(p(z_j|n'_1),\ldots,p(z_j|n'_{\lvert \chd(n)\rvert})), \label{eq:deltaIz}\\
\Delta I_X(n) &= p(n) H(\Pi), \label{eq:deltaIx}
\end{align}
where \(p(y_i|n'_u)\) for \(n'_u\in\chd(n)\) are recursively computed via
\begin{align}
p(y_i|n'_u) &= \sum_{n''\in\chd(n'_u)} \Pi_{n''} p(y_i | n''),~i\in\calK_Y, \label{eq:parentChildRelConditional} \\
p(z_j|n'_u) &= \sum_{n''\in\chd(n'_u)} \Pi_{n''} p(z_j | n''),~j\in\calK_Z,\label{eq:parentChildIrrelConditional} 
\end{align}
and \(\Pi\in\Re_{+}^{\lvert \chd(n'_u)\rvert}\) has entries \(\Pi_{n''} = \nicefrac{p(n'')}{p(n'_u)}\).
Once the G-values~\eqref{eq:fullGfunction} are known from an inverse breadth-first node traversal of $\T_\W$, G-tree search considers nodes \(n\in\Nint(\T_\W)\) in a top-down manner to determine whether or not they should be part of the solution to~\eqref{eq:GeneralTreeSrchProblem}.
See~\cite{larsson2022generalized} for more details regarding the G-tree search algorithm.
Careful inspection of~\eqref{eq:fullGfunction} and~\eqref{eq:deltaIy}-\eqref{eq:deltaIx} reveals that G-tree search depends on \(p(x,y_{1:|\calK_{Y}|}, z_{1:|\calK_{Z}|})\) only via the distributions \(p(y_i|x)\), \(p(z_j|x)\) and \(p(x)\).
Thus, we assume that a distribution \(p(x)\) over leaf nodes is provided, and determine \(p(y_i|x)\) and \(p(z_j|x)\) from semantic octree perception data.
Our solution  consists of two phases: (i) \emph{the update pass}: inserts or updates nodes in the current octree based on semantic perception data, and (ii) \emph{the octree compression pass}: executes G-tree search to compress the environment representation that is built as part of phase 1.
Next, we describe the tree-building process before delineating how \(p(y_i|x)\) and \(p(z_j|x)\) are obtained from perception data.
%

%%%%%%%
\subsection{Updating Tree Nodes from Perceptual Data}\label{subsec:semanticTree}

We adopt the semantic model proposed in~\cite{semantic-octomap} to build a hierarchical Bayesian multi-class octree representation of the world.
The tree-building algorithm builds the finest resolution octree \(\T_\W\) from observations (see~\figref{fig:octree_update}), and maintains a truncated probability distribution over semantic classes, represented by a random variable \(S:\Omega\to\calK\) for each leaf node $x\in \Nlf(\T_\W)$.
In more detail, given a node \(x\in \Nlf(\T_\W)\), if \(k\in \calK_3(x) \cup \{0\}\) then \(p(S = k|x)\) is provided by the octree, where \(\calK_3(x)\) is the set of three most likely classes of node $x$.
The octree also stores a forth entry, corresponding to \(p(S\in\calK \setminus (\calK_3(x) \cup \{0\}) |x)\), which is the probability of the event that the node \(x\) belongs to a semantic class other than the three most likely or free space.
Furthermore, to reduce the memory required to store the map, the algorithm will prune nodes whose children all have identical multi-class probability distributions.
\begin{figure}[t]
    \centering
    \includegraphics[width=\linewidth]{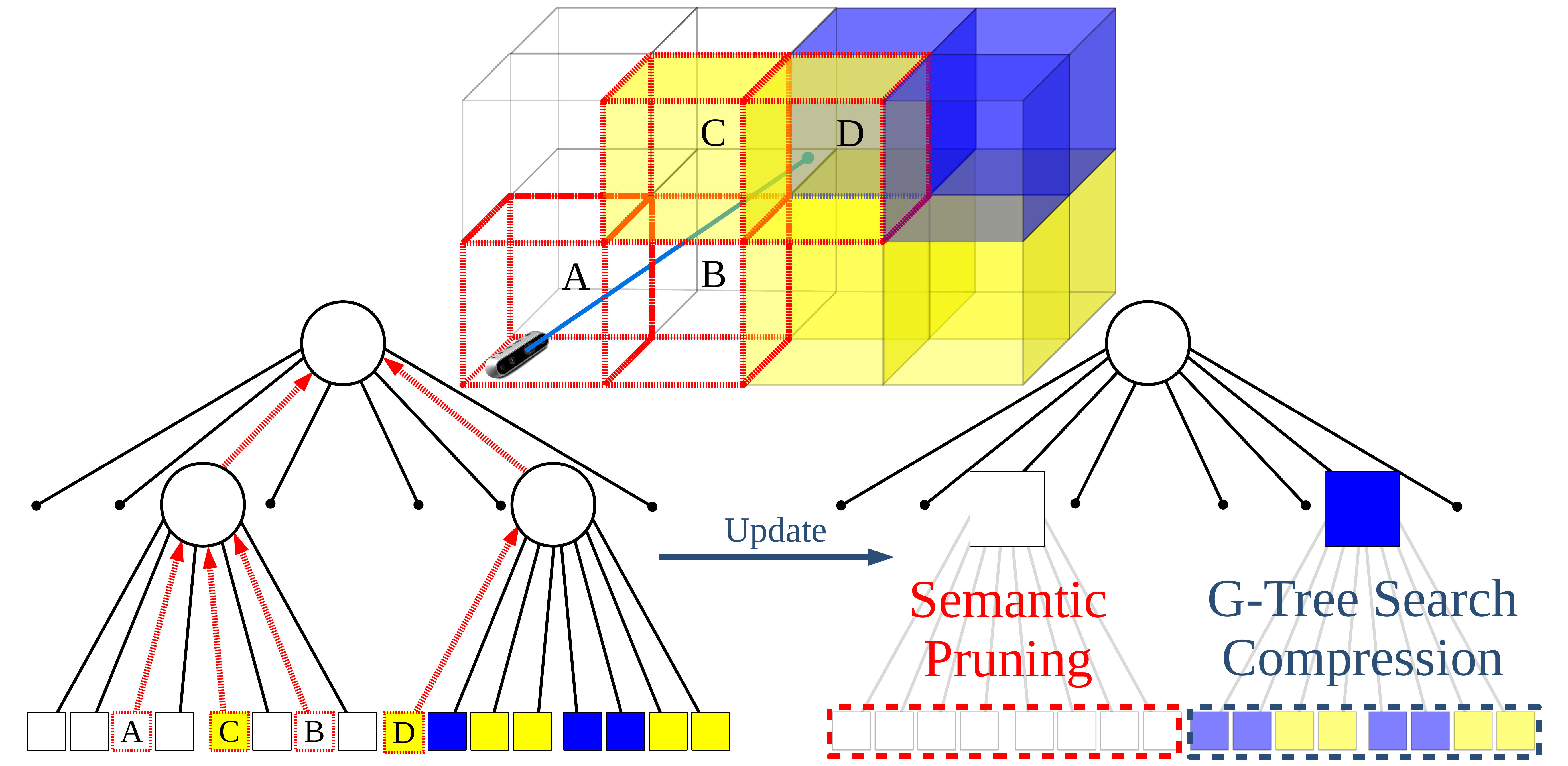}
    \caption{Semantic octree update from a new observation indicating cell D is blue and C is free space. Each leaf node with color-encoded object class is depicted as a square, inner nodes are represented by circles, and unexplored cells are shown as dots. Multi-class probabilities and G-values along the paths from leaf nodes to the root, highlighted in red, need to be updated due to the new perceptual data.}
    \label{fig:octree_update}
\end{figure}
%%

%%%%%%%%%%%%%%%%%%%%%%%%%
%%%%%%%%%%%%%%%%%%%%%%%%%
\subsection{Extracting Semantic Information from the Octree Data Structure}\label{subsec:extracSemanticDistLeafs}

For any \(i\in\calK_Y\), the conditional distribution \(p(y_i|x)\)  is derived from the semantic octree according to \(p(y_i=1|x) = p(S=i|x)\), with an analogous expression for \(p(z_j|x)\), \(j\in\calK_Z\).
In practice, obtaining the conditional distribution \(p(y_i|x)\) (resp. \(p(z_j|x)\)) presents a challenge, since the tree-building algorithm only maintains the truncated semantic probabilities.
Thus, we cannot directly obtain \(p(y_i|x)\) or \(p(z_j|x)\) from the semantic octree structure, as the $3$ most likely semantic classes may differ from node to node.
Instead, we must generate the distribution \(p(s|x)\) over \emph{all} semantic classes.
To this end, for each leaf node \(x\in\Nlf(\T_\W)\), we collect the truncated semantic probabilities and uniformly distribute the probability \(p(S\in\calK \setminus (\calK_3(x) \cup \{0\}) |x)\) among the $K - 3$ outstanding classes.
Once the distributions \(p(y_i|x)\), \(p(z_j|x)\), and \(p(x)\) are known for the leaf nodes of the octree \(\T_\W\), we may apply the recursive relations~\eqref{eq:parentPx} and~\eqref{eq:parentChildRelConditional}-\eqref{eq:parentChildIrrelConditional} to update the semantic distributions for all interior nodes.
It is worth mentioning that each node \(n\in\Nint(\T_\W)\) may not always have a full set of children since the tree-building algorithm instantiates nodes only for the observed locations in the environment.
To see why missing children pose a challenge, assume node \(n\in\Nint(\T_\W)\) has only a single child; that is \(\chd(n) = \{n'\}\).
From~\eqref{eq:parentPx} and~\eqref{eq:deltaIy}-\eqref{eq:parentChildIrrelConditional}, we see that if \(\chd(n) = \{n'\}\), then we have \(\Delta I_X(n) = 0\), \(\Delta I_{Y_i}(n) = 0\) and \(\Delta I_{Z_j}(n) = 0\); so no information is lost in aggregating the child node \(n'\) to its parent \(n\).
In order to remedy this issue, we account for absent children by instantiating a maximum entropy distribution (i.e., uniform over $\calK$) and a value of \(p(n)\) for nodes not represented in the octree, from which \(p(y_i|x)\) and \(p(z_j|x)\) are obtained analogously to the existing leaf nodes.
Importantly, the recursive structure of~\eqref{eq:parentPx} and \eqref{eq:parentChildRelConditional}-\eqref{eq:parentChildIrrelConditional} imply that \(p(y_i|n'),~p(z_j|n')\) and \(p(n')\) for missing child nodes \(n'\) need only be instantiated to compute the values of their parent \(n\) and not explicitly represented in the octree.
%

%%%%%%%%%%%%%%%%%%%
\subsection{Updating G-values from Local Tree Information}\label{subsec:updateGvals}

From~\eqref{eq:parentChildRelConditional}-\eqref{eq:parentChildIrrelConditional} we note that computing \(p(y_i|n)\) and \(p(z_j|n)\) for any \(n\in\Nint(\T_\W)\) can be done from knowledge of \(p(y_i|n')\) and \(p(z_j|n')\) for \(n'\in\chd(n)\).
Moreover, these updates do not require \(p(x)\) to be a valid probability distribution, since~\eqref{eq:parentChildRelConditional}-\eqref{eq:parentChildIrrelConditional} depend only the \emph{relative} weights \(\nicefrac{p(n')}{p(n)}\).
However, this structure is not shared by the G-function~\eqref{eq:fullGfunction}, since the latter has an explicit dependence on \(p(x)\).
It is therefore of interest to investigate if characteristics of the G-function, or some equivalent, can be expressed in terms of only relative weights.
To this end, we define \(G_\Pi:\N(\T_\W)\times\Renp\times\Remp\times\Re_+ \to \Re_+\) according to: if \(n\in\Nint(\T_\W)\) and \(p(n) > 0\) then
\begin{align}
	&G_\Pi(n;\beta,\gamma,\alpha) \nonumber \\
	&~~~=\max\{\sum_{i\in\calK_Y}\beta_i \JS^{Y_i}(n) - \sum_{j\in\calK_Z}\gamma_i\JS^{Z_j}(n) - \alpha H(\Pi) \nonumber \\
	&~~~~~~~~~~~~~~~~~~~ + \sum_{n'\in\chd(n)}\Pi_{n'}G_\Pi(n';\beta,\gamma,\alpha),~0\}, \label{eq:localGfunction}
\end{align}
and \(G_\Pi(n;\beta,\gamma,\alpha) = 0\) otherwise, where for \(i\in\calK_Y\), \(\JS^{Y_i}(n) = \JS(p(y_i|n'_1),\ldots,p(y_i|n_{\lvert \chd(n)\rvert})\), \(n'_u\in\chd(n)\) with \(\JS^{Z_j}(n)\) defined analogously.
This brings us to the following result.
\begin{proposition}\label{prop:localGFunction}
	Let \(n\in\N(\T_\W)\).
	Then \(G_{\Pi}(n;\beta,\gamma,\alpha) > 0\) if and only if \(G(n;\beta,\gamma,\alpha) > 0\).
\end{proposition}
\begin{proof}
	The proof is given in the Appendix.
\end{proof}
As a result of Proposition~\ref{prop:localGFunction}, we may predicate our pruning on the function \(G_{\Pi}(n;\beta,\gamma,\alpha)\) in place of \(G(n;\beta,\gamma,\alpha)\) without sacrificing any of the theoretical guarantees of the G-tree search method.
In the next section, we present the joint tree-building and compression algorithm. 
%

%%%%%%%%%%%%%%%%%%%
\subsection{The Joint Tree-Building \& Compression Algorithm}\label{subsec:jointBuildCompressAlgo}

The joint semantic octree-building and compression framework is shown in~\algoref{alg:jointTreeBuildandCompress}.
%
%%%%%%%%%%%%%%%
%%%%%%%%%%%%%%%
% Subroutine for creating/extracting a full-length probability vector 
\begin{algorithm}[t]
	\SetKwInOut{Input}{input}\SetKwInOut{Output}{output}
	\Input{Semantic point cloud \(\mathcal P\), G-tree search weights \((\beta,\gamma,\alpha)\in\Renp\times\Remp\times\Re_+\).}
	\Output{Compressed octree representation \(\T^*\) of \(\W\).}
	%
	%%%%%%%%%%%%%%%%%
	% start if block
	\If{point cloud data recieved}{ \label{algLn:startSemanticTreeUpdate}
		\(x \leftarrow \text{\texttt{createOrUpdateNode}}(\mathcal P)\)\label{algLn:insertLfNd}\;
		\(n \leftarrow x\)\;
		\(G_{\Pi}(n;\beta,\gamma,\alpha) = 0\)\label{algLn:gValBoundary}\;
		\While{\(\text{\texttt{Parent}}(n) \neq \varnothing\)}
		{\label{algLn:startTreeRecurse}
			\(\bar n \leftarrow \text{\texttt{Parent}}(n)\)\;
			\eIf{\(\chd(\bar n)\subseteq \Nlf(\T_\W)\)}{\label{algLn:ifParentOfLeaf}
				\((p(y_i|\bar n),p(z_j|\bar n),p(\bar n)) \leftarrow \text{\texttt{getDist}}(\bar n)\)\;
			}
			{
				\((p(y_i|\bar n),p(z_j|\bar n),p(\bar n)) \leftarrow \text{\texttt{chdDist}}(\bar n)\)\;
			}
			
			\(G_{\Pi}(\bar n;\beta,\gamma,\alpha) \leftarrow \text{\texttt{updateGvals}}(\beta,\gamma,\alpha)\)\label{algLn:updateGvals}\;
			\(n \leftarrow \bar n\)\label{algLn:endSemanticTreeUpdate}\;
		\label{algLn:endTreeRecurse}}
		\(\T^* \leftarrow \text{\texttt{GtreeSearch}}(n_{\mathsf R})\)\label{algLn:treeCompress}\;
	}
	%
	% end of if block
	%%%%%%%%%%%%%%%%%
	\KwRet{\(\T^*\)}
	\caption{Joint semantic-tree construction and compression}
	\label{alg:jointTreeBuildandCompress}
\end{algorithm}
%%%%%%%%%%%%%%%
%%%%%%%%%%%%%%%
%
%%%%
% Large figure showing outdoor test environment (for results section)
\begin{figure*}[t]
	\centering
	\subfloat[\label{fig:envSum1}]{\includegraphics[width=0.2437\textwidth]{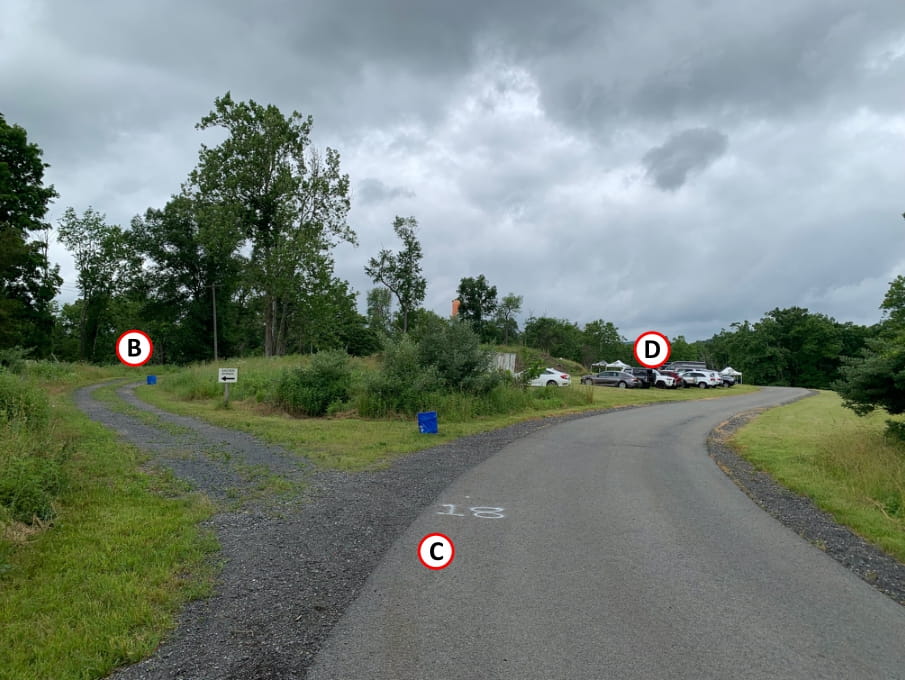}} 
	\hfill
	\subfloat[\label{fig:envSum2}]{\includegraphics[width=0.245\textwidth]{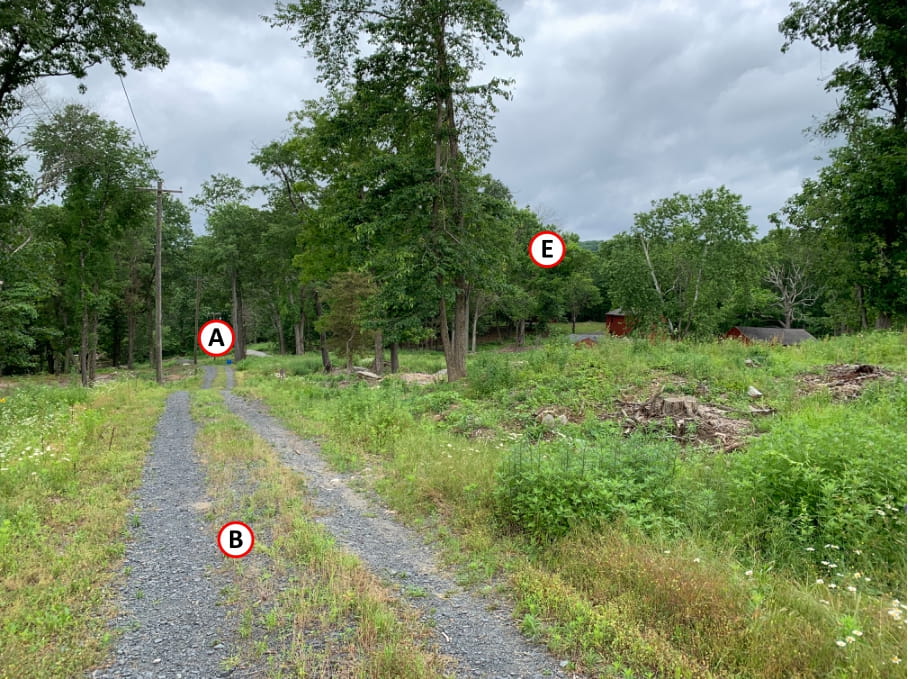}}
	\hfill
	\subfloat[\label{fig:envSum3}]{\includegraphics[width=0.244\textwidth]{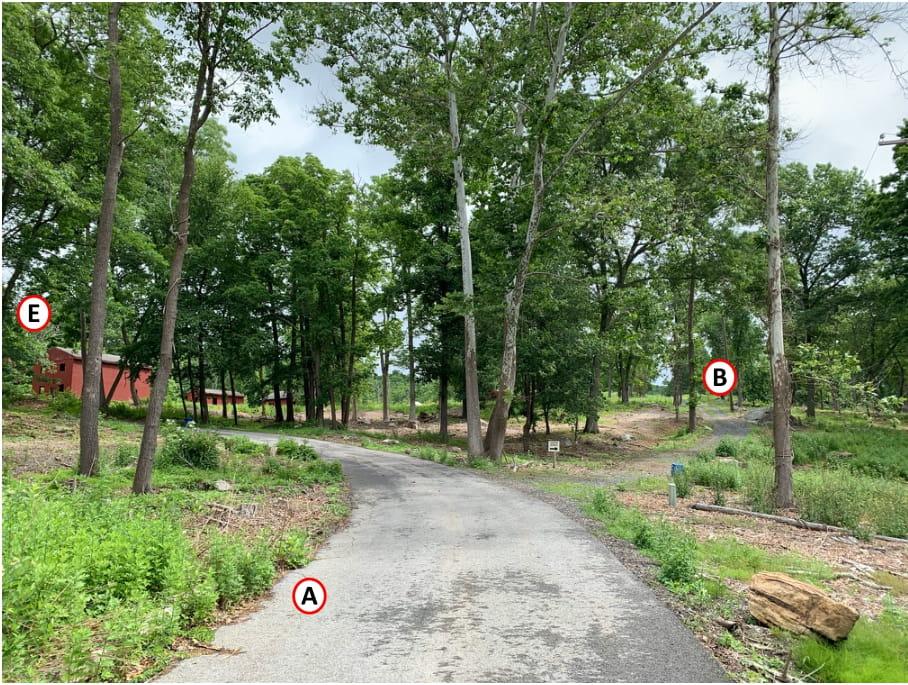}}
	\hfill
	\subfloat[\label{fig:envSum4}]{\includegraphics[width=0.245\textwidth]{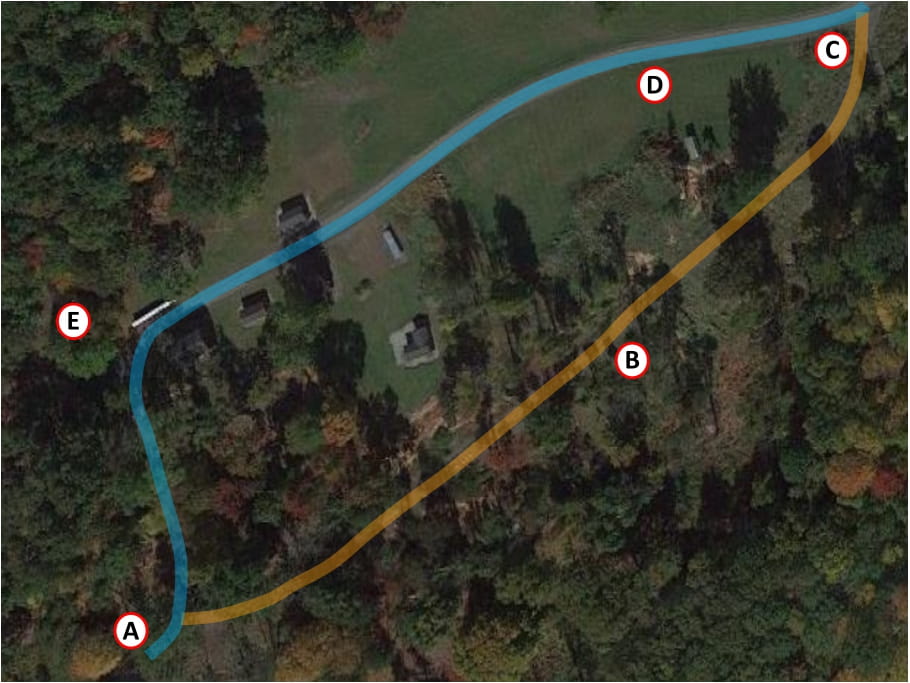}}
	\caption{Outdoor testing environment. (a)-(c) ground-level view from various perspectives. (d) top-down view, with asphalt and dirt/gravel road labeled in cyan and orange, respectively.}
	\label{fig:envSummary}
\end{figure*}
%%%%
%
The update process (i.e., lines~\ref{algLn:startSemanticTreeUpdate}-\ref{algLn:endSemanticTreeUpdate}) is triggered by the availability of new semantic point-cloud data.
Moreover, the function \(\text{\texttt{createOrUpdateNode}}(\mathcal P)\) either creates or updates the semantic information of an existing leaf node of the octree \(\T_\W\) from the new data (see~\figref{fig:octree_update}).
Since incoming information is always inserted as a leaf of \(\T_\W\), we, in line~\ref{algLn:gValBoundary}, set the leaf-node condition for the G-values.
We then traverse the octree bottom-up in lines~\ref{algLn:startTreeRecurse}-\ref{algLn:endTreeRecurse}, visiting the sequence of node parents until the root node is reached, updating G-values and semantic distributions along the way (see~\figref{fig:octree_update}).
If the node \(\bar n\) is a parent of a leaf node, such that line~\ref{algLn:ifParentOfLeaf} is true, then we extract the full semantic distribution as detailed by Section~\ref{subsec:extracSemanticDistLeafs} for the children \(x\in\chd(\bar n)\).
This is done by the routine \(\text{\texttt{getDist}}(\bar n)\), before applying the recursive updates~\eqref{eq:parentChildRelConditional}-\eqref{eq:parentChildIrrelConditional}. 
In contrast, if \(\bar n\) is not a parent of a leaf, then \(p(y_i|n')\) and \(p(z_j|n')\) are known for all \(n'\in\chd(\bar n)\), and thus, \(\text{\texttt{chdDist}}(\bar n)\) retrieves the distributions of the children and applies~\eqref{eq:parentChildRelConditional}-\eqref{eq:parentChildIrrelConditional}, before updating the G-values in line~\ref{algLn:updateGvals} according to~\eqref{eq:localGfunction}.
We then call the tree-compression algorithm in line~\ref{algLn:treeCompress} by passing the root node of \(\T_\W\), denoted \(n_{\mathsf R}\), to G-tree search.
In the language employed at the start of this section, lines~\ref{algLn:insertLfNd}-\ref{algLn:endTreeRecurse} comprise the update pass (phase 1) and line~\ref{algLn:treeCompress} constitutes the octree compression step (phase 2).
Note that the bottom-up recursion defined by lines~\ref{algLn:startTreeRecurse}-\ref{algLn:endTreeRecurse} of~\algoref{alg:jointTreeBuildandCompress} is made possible by the following two observations.
First, at each time instance for which point cloud information is available, the semantic tree-building algorithm inserts (or updates) exactly one leaf node of the current octree representation of \(\W\).
Secondly, the function \(G_{\Pi}(n;\beta,\gamma,\alpha)\) and distributions \(p(y_i|n)\), \(p(z_j|n)\) and \(p(n)\) can all be updated from, and only depend on, immediate child information.
Thus, the nodes traversed by the recursion in lines~\ref{algLn:startTreeRecurse}-\ref{algLn:endTreeRecurse} of~\algoref{alg:jointTreeBuildandCompress} are precisely those whose G-values and semantic distributions are effected by the new perceptual data received (see~\figref{fig:octree_update}).
In contrast, the G-tree search method requires an inverse breadth-first recursion to compute G-values, and does not consider the availability of new semantic data.
%

%%%%%%%%%%%%%%%%%%%%%%%%%%%%%%%%%%%%%%%%%%%%%%%%%%%%%%%%%%%%%%%%%%%%%%%%%%%%%%%%
\section{Real-world Experiments} \label{sec:experiments}

%%%%%%%%%%%%%%%%%%
% figure showing full semantic environment and sample abstraction:
\begin{figure}[t]
    \centering
    \includegraphics[width=\linewidth]{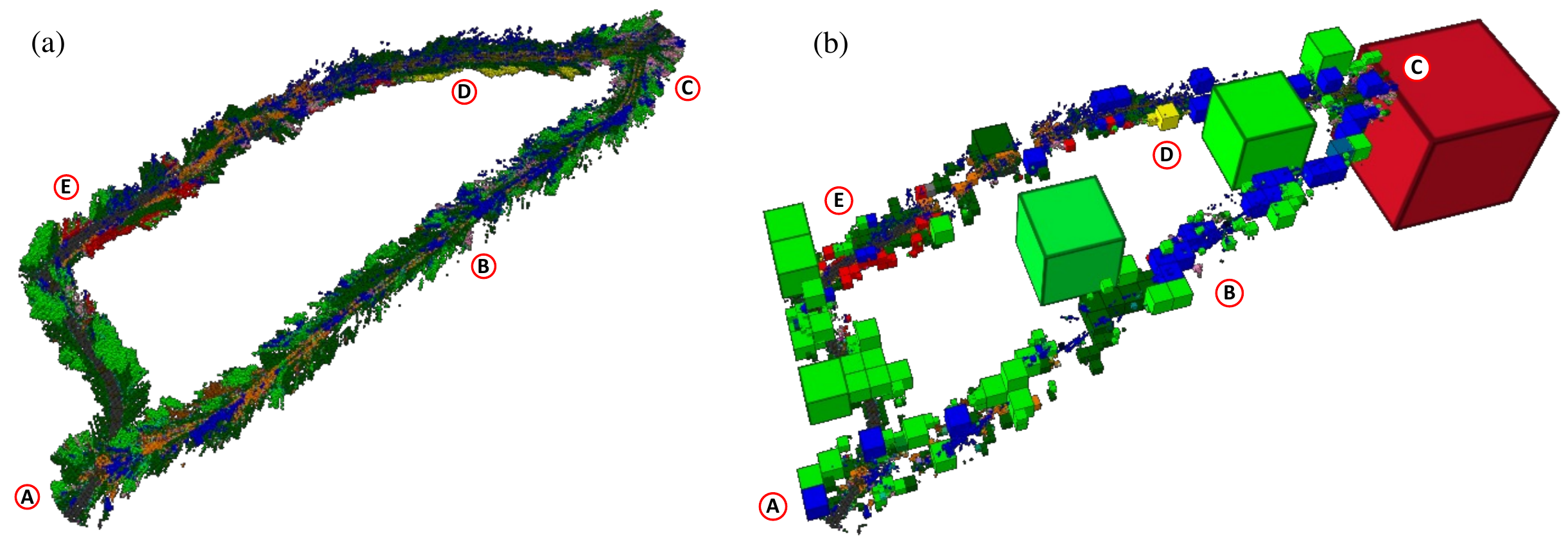}
    \caption{(a) finest-resolution semantic octree representation build from point-cloud data. (b) Compressed multi-resolution octree obtained from the framework in Section~\ref{sec:solnApproach}.}
    \label{fig:envTreeExample}
\end{figure}
%%%%%%%%%%%%%%%%%%
%

%
In this section, we present and discuss results obtained from field-test experiments performed in the outdoor environment shown in~\figref{fig:envSummary}.
The environment contains a number of elevation changes, and a combination of urban and rural elements; See~\mfigref{fig:envSum1}{fig:envSum3}. 
The FCHarDNet classifier~\cite{fchardnet} pre-trained on RUGD dataset~\cite{rugd} is employed for semantic classification, providing 24 possible categorizations (i.e., \(\calK=\{0,\ldots,24\}\)).
Results are presented for two scenarios: (i) to study the semantic tree-building and compression algorithm detailed in Section~\ref{sec:solnApproach}, and (ii) to demonstrate how the abstractions can be employed in semantically-informed planning problems.
% 

%%%%%%%%%%%%%%%%%%%%%%%%%%%%%%%%%%%%%%%%%%%%
%%%%%%%%%%%%%%%%%%%%%%%%%%%%%%%%%%%%%%%%%%%%
\subsection{Semantic Perception \& Tree Compression}\label{subsec:resultsPercepCompress}

A finest-resolution semantic octree \(\T_\W\) built from sensor data is shown alongside a compressed octree in~\figref{fig:envTreeExample}. 
The abstraction shown in~\figref{fig:envTreeExample}(b) is created by specifying asphalt as a relevant semantic class, while setting the classes of grass and trees to irrelevant (i.e., undesired). 
Thus, we see that regions of the map that contain paved road (e.g., A to E) retain high resolution as compared with areas that contain little to no paved road and greater amount of grass and trees (e.g., the portion from A to B to C). 
Furthermore, classes that are not relevant nor undesired (e.g., sky), are aggregated whenever doing so does not contribute to a loss of relevant information. 
%

%%%%%%%%%%%%%%%%%%
% figure showing bar graph of information and leaf nodes:
\begin{figure}[t]
    \centering
    \includegraphics[width=0.9\linewidth]{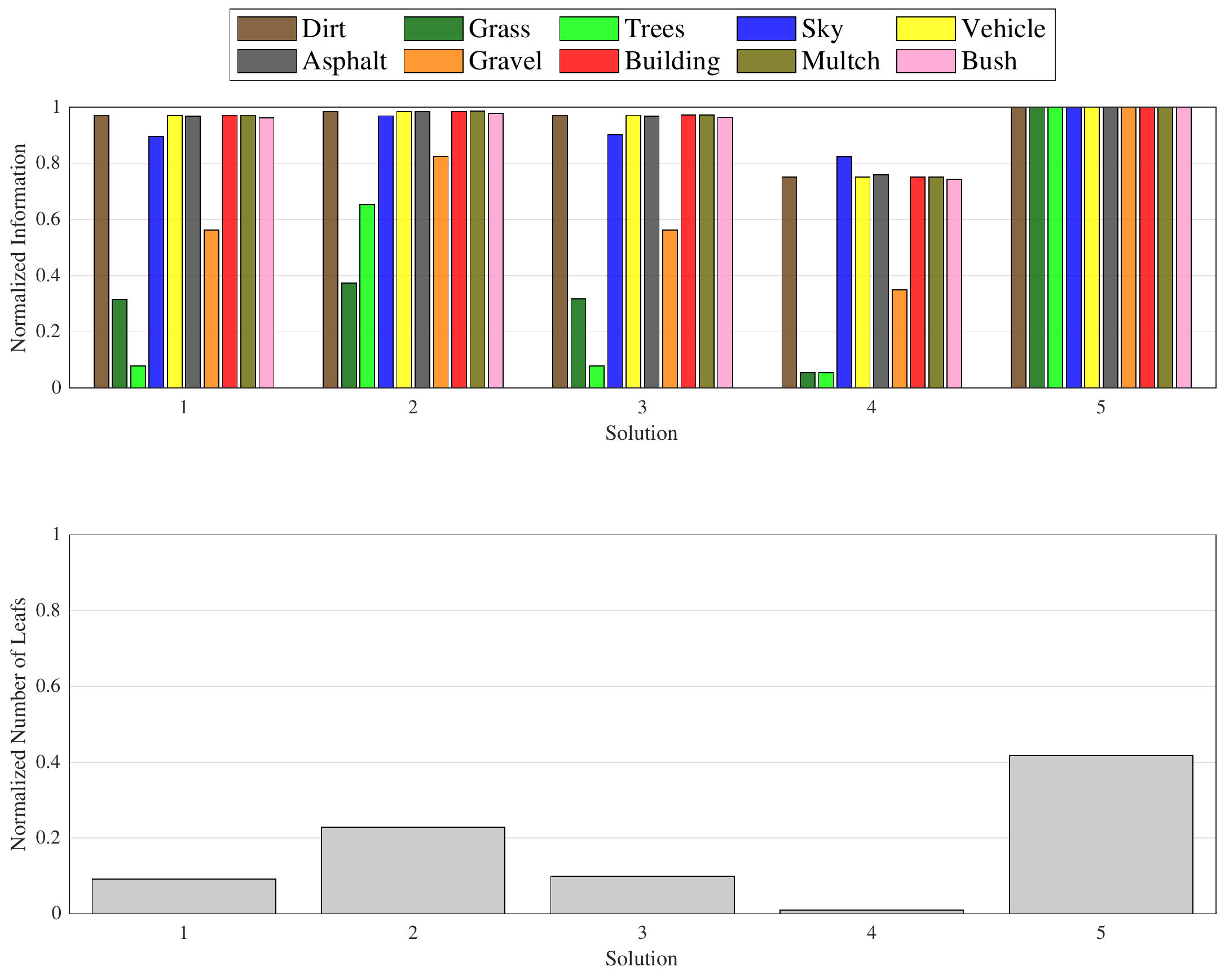}
    \caption{(top) Normalized degree of information retained for each semantic class as a function of the G-tree search weights (solution number). (bottom) Normalized number of leaf nodes of the compressed octree.}
    \label{fig:combInfoAndLeafs}
\end{figure}
%%%%%%%%%%%%%%%%%%

%%%%%%%%%%%%%%%%%%
% figure showing pie chart of relative weights:
\begin{figure}[t]
    \centering
    \includegraphics[width=0.65\linewidth]{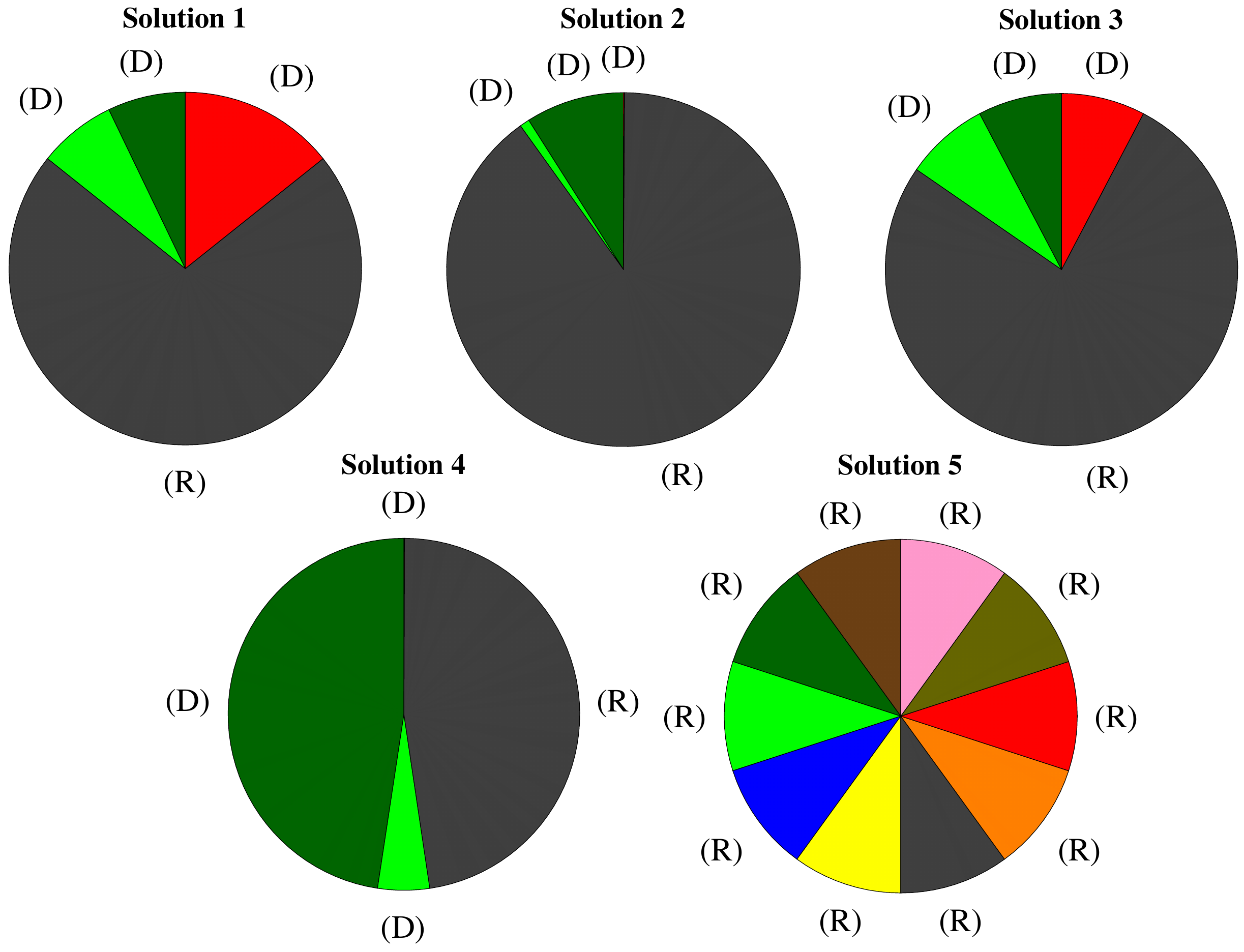}
    \caption{Weights for G-tree search. Colors correspond with the semantic class coloring, (R) is a relevant class, and (D) is an irrelevant one. In solutions 2 and 4, the (D) near the top corresponds to the red class (building).}
    \label{fig:relativeWeights}
\end{figure}
%%%%%%%%%%%%%%%%%%

%
Next, consider~\figref{fig:combInfoAndLeafs}, which shows the normalized degree of information retained regarding each of the visible semantic classes in~\figref{fig:envTreeExample}(a) as a function of the G-tree search weights (solution number) depicted in~\figref{fig:relativeWeights}.
By comparing~\twoFigRef{fig:combInfoAndLeafs}{fig:relativeWeights} we make a few observations. 
First, observe that solution 5 considers all semantic classes as relevant, and thus the abstraction returned by G-tree search retains all information.
Secondly, we see the impact of the weights on the octree solution by considering solutions 2 and 4.
To this end, note that solution 4 contains considerably less information regarding the tree and grass classes as compared with solution 2, which occurs since the 4\(^\text{th}\) solution penalizes the retention of both grass and trees (dark and light green, respectively) to a much higher degree than solution 2. 
Notice also that solution 4 contains less information regarding \emph{all} classes compared with solution 2, since the priority of information removal outweighs the importance of information retention resulting with a more compressed octree, confirmed by~\figref{fig:combInfoAndLeafs}(bottom).
Moreover observe that the G-tree search method is able to find a compressed octree that retains most of the semantic information, as seen by solution 5.
To understand why this is the case, recall that node information \(\Delta I_{Y_i}(n)\) and \(\Delta I_{Z_j}(n)\) in the G-tree search method is quantified via the JS-divergence.
Thus, nodes with smaller values of \(\Delta I_{Y_i}(n)\) and \(\Delta I_{Z_j}(n)\) imply that the semantic distributions \(p(y_i|n)\) and \(p(z_j|n)\) are more similar to the distributions \(p(y_i|n')\) and \(p(z_j|n')\) of their children \(n'\in\chd(n)\) as compared with nodes having greater values of \(\Delta I_{Y_i}(n)\) and \(\Delta I_{Z_j}(n)\).
Contrast this with the ad-hoc pruning rule employed by the tree-building process of Section~\ref{subsec:semanticTree} (see~\figref{fig:octree_update}), which prunes nodes based on thresholds placed individual elements of the distributions \(p(y_i|n)\) and \(p(z_j|n)\). 
Consequently, it may happen that nodes pass the ad-hoc pruning test but contain little to no information due to the similarity of the semantic probability distributions of its immediate children. 
The G-tree search method is able to exploit these redundancies, leading to the high degrees of compression seen in~\figref{fig:combInfoAndLeafs}.
%

%%%%%%%%%%%%%%%%%%%%%%%%%%%%%%%%%%%%%%%%%%%%
%%%%%%%%%%%%%%%%%%%%%%%%%%%%%%%%%%%%%%%%%%%%
\subsection{Semantic Perception, Abstraction and Planning} \label{subsec:resultsCompressAndPlan}

Lastly, we discuss how the perception-abstraction framework developed in Section~\ref{sec:solnApproach} can be employed to construct a more (semantically) informed graph than conventional techniques (e.g., Halton sequence~\cite{Halton1964}) for use in colored graph-search planning algorithms.
We consider the Class-Ordered A* (COA*)~\cite{Lim2021} algorithm, which computes a semantically-informed path and allows for both desired (i.e., relevant) and undesired semantic classes to be specified. 
The COA* algorithm searches a weighted colored (semantic) graph to find the shortest path that contains the least number of edges in unwanted classes~\cite{Lim2021}.
Vertices in the search graph correspond to a two-dimensional coordinate and heading configuration of a non-holonomic ground robot, that is \((x, y, \theta)\), and edges represent the Reeds-Shepp curve~\cite{Reeds1990}.
Thus, all paths in the graph are dynamically feasible.
Graph vertices are classified (i.e., given a color) according to the semantic information contained in the (semantic) octree node corresponding to the volumetric region containing the \((x,y)\)-coordinate which the vertex represents.
Shown in~\figref{fig:abstracted_planning} are example paths obtained from employing COA* on colored (semantic) graphs generated by Halton sequences and from the semantic octree compression algorithm developed in Section~\ref{sec:solnApproach}, with road as a relevant (preferred) class for both planning and abstraction.
Note that the graph generated from a Halton sequence is agnostic to semantic information.
To provide quantitative results, we averaged planning results over 50 search instances for both methods. 
From our study, we observed that the optimal path was found about 10\% percent faster on average in the semantically informed (compressed) graph generated by the framework from Section~\ref{sec:solnApproach}.
Moreover, we noted that the standard deviation of the planning time was reduced by 60\% percent when the the semantically informed (compressed) graph was employed for semantic planning.
%
%%%%%%
\begin{figure}[t]
	\centering
	\subfloat[\label{fig:non_abstract_planning}]{\includegraphics[width=0.4\linewidth]{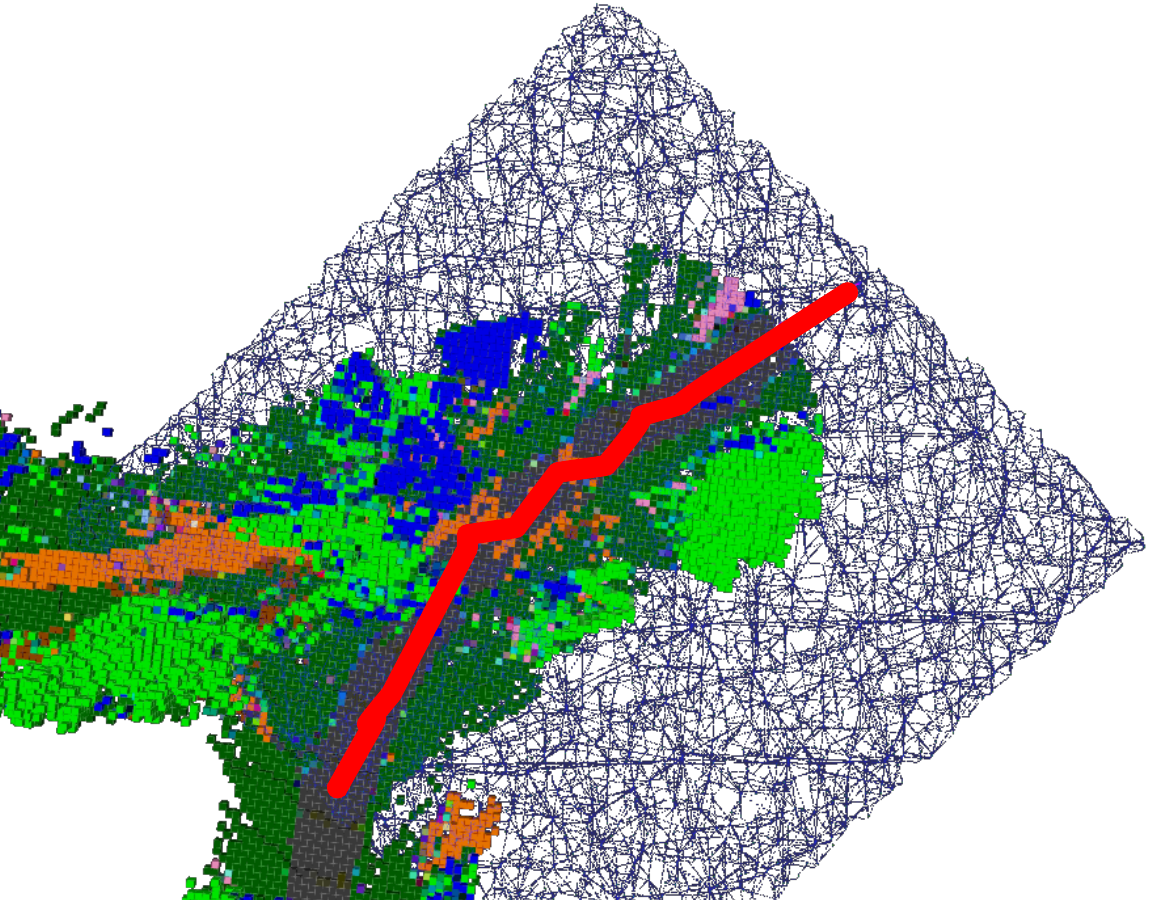}}
	\hfil
	\subfloat[\label{fig:abstract_planning}]{\includegraphics[width=0.4\linewidth]{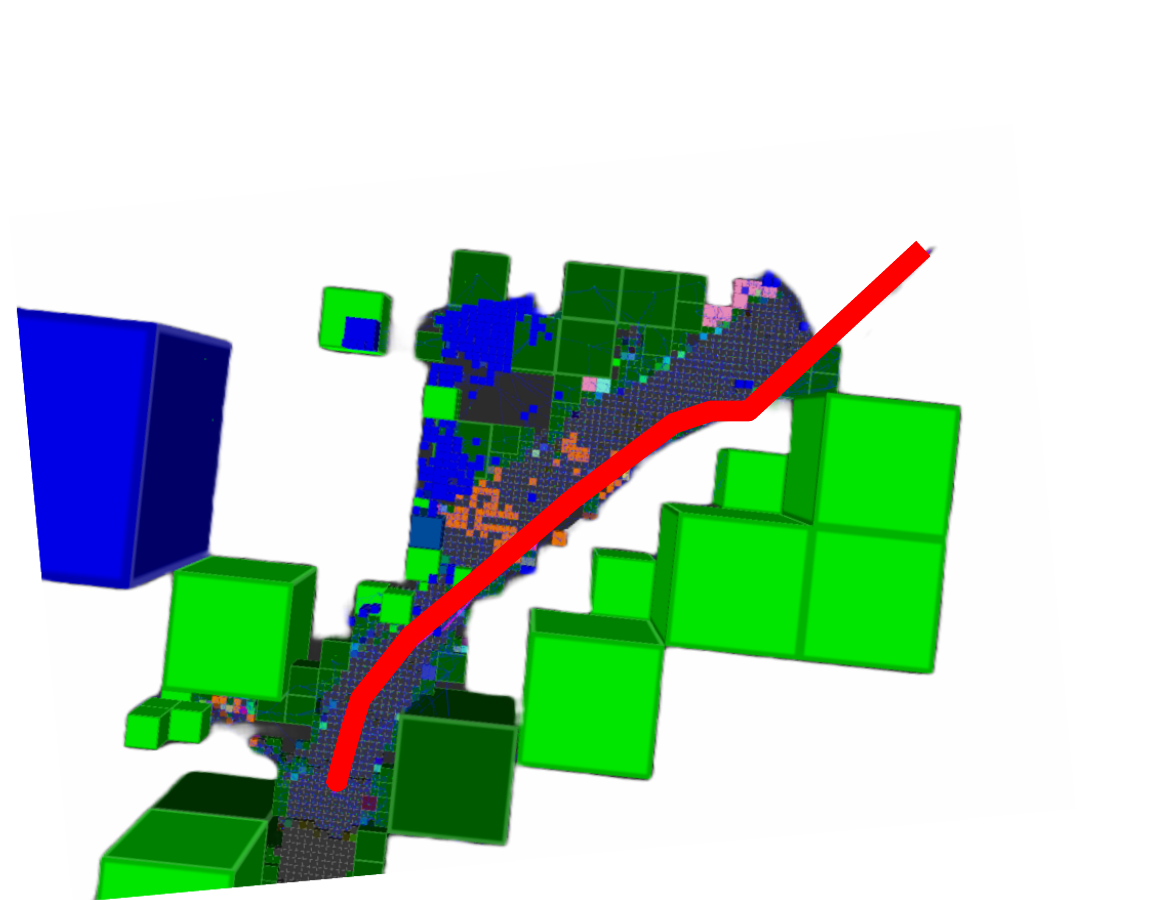}}
    \caption{(a) The optimal path found by COA* in the graph constructed with a Halton sequence. (b) The optimal path found by COA* in the graph constructed based on the compressed octomap nodes.}
    \label{fig:abstracted_planning}
\end{figure}
%%%%%%

%%%%%%%%%%%%%%%%%%%%%%%%%%%%%%%%%%%%%%%%%%%%%%%%%%%%%%%%%%%%%%%%%%%%%%%%%%%%%%%%
\section{Conclusions} \label{sec:conclusions}

In this paper, we considered the development of a joint semantic mapping and compression framework that simultaneously builds and compresses $3$-D probabilistic semantic octree representations.
Our framework consists of two parts: a Bayesian multi-class tree-building framework and information-theoretic tree-compression.
The developed framework is fully probabilistic and allows multi-resolution abstractions to be tailored to task-relevant and task-irrelevant semantic classes and information.
To demonstrate our approach, we compress large semantic maps built from real-world sensor data, and show how the abstractions can be used to improve the performance of planning algorithms over colored (semantic) graphs.
%

%%%%%%%%%%%%%%%%%%%%%%%%%%%%%%%%%%%%%%%%%%%%%%%%%%%%%%%%%%%%%%%%%%%%%%%%%%%%%%%%
\section*{Appendix: Proof of Proposition~\ref{prop:localGFunction}}\label{sec:appendix}

\begin{proof}
To prove the proposition, we show that \(G(n;\beta,\gamma,\alpha) = p(n)G_{\Pi}(n;\beta,\gamma,\alpha)\) for all \(n\in\N(\T_\W)\).
There are two cases to consider for any \(n\in\Nint(\T_\W)\): \(p(n) = 0\) and when \(p(n) > 0\).
The first of these cases is straightforward, since by definition \(G(n;\beta,\gamma,\alpha) = 0\) and \(G_{\Pi}(n;\beta,\gamma,\alpha) = 0\).
Thus, \(G(n;\beta,\gamma,\alpha) = p(n)G_{\Pi}(n;\beta,\gamma,\alpha) = 0\) when \(p(n) = 0\).
We now show \(G(n;\beta,\gamma,\alpha) = p(n) G_{\Pi}(n;\beta,\gamma,\alpha)\) for any \(n\in\N(\T_\W)\) for which \(p(n) > 0\).
The proof is given by induction.
Consider any \(n\in\N(\T_\W)\) that is a parent of a leaf, then 
\begin{align*}
	&G(n;\beta,\gamma,\alpha) \\
	&= p(n) \max\{\sum_{i}\beta_i\JS^{Y_i}(n)-\sum_{j}\gamma_j\JS^{Z_j}(n) - \alpha H(\Pi),~0\},\\
	&= p(n) G_{\Pi}(n;\beta,\gamma,\alpha),
\end{align*}
which follows from the properties of maximum since \(p(n) \geq 0\).
Now consider any \(n\in\N_{k}(\T_\W)\) (i.e., any node at depth \(k\)), \(k \geq 1\), and assume the hypothesis holds for all \(n'\in\N_{k+1}(\T_\W)\) for which \(p(n') > 0\).
Then,
\begin{align*}
	&G(n;\beta,\gamma,\alpha) \\
	&~~~~= p(n) \max\{\sum_{i}\beta_i\JS^{Y_i}(n)-\sum_{j}\gamma_j\JS^{Z_j}(n) - \alpha H(\Pi)\\
	&~~~~~~~~~~~~~~~~~~~~~~~~~~~~~ + \frac{1}{p(n)}\sum_{n'\in \mathcal S} G(n';\beta,\gamma,\alpha),~0\},
\end{align*}
where \(\mathcal S = \{n'\in\chd(n): p(n') > 0\}\), \(\mathcal S \subseteq \N_{k+1}(\T_\W)\).
The quantity within the max operator can be written as
\begin{align*}
	&\max\{\sum_{i}\beta_i\JS^{Y_i}(n)-\sum_{j}\gamma_j\JS^{Z_j}(n) - \alpha H(\Pi)\\
	&~~~~~~~~~~~~~~~~~~~~~~~~~~~~~ + \frac{1}{p(n)}\sum_{n'\in\mathcal S} G(n';\beta,\gamma,\alpha),~0\},\\
	&= \max\{\sum_{i}\beta_i\JS^{Y_i}(n)-\sum_{j}\gamma_j\JS^{Z_j}(n) - \alpha H(\Pi)\\
	&~~~~~~~~~~~~~~~~~~ + \frac{1}{p(n)}\sum_{n'\in \chd(n)} p(n')G_{\Pi}(n';\beta,\gamma,\alpha),~0\},
\end{align*}
where the equality holds from the induction hypothesis and since, for \(n' \in \{\bar n \in \chd(n) : \bar n \not\in \mathcal S\}\), we have \(p(n')G_{\Pi}(n';\beta,\gamma,\alpha) = 0\) by definition, leading to \(G(n;\beta,\gamma,\alpha)=p(n)G_{\Pi}(n;\beta,\gamma,\alpha)\).
To show the proposition, we prove if \(G(n;\beta,\gamma,\alpha) > 0\) then \(G_{\Pi}(n;\beta,\gamma,\alpha) > 0\) and its converse.
Pick any \(n\in\N(\T_\W)\) and assume \(G(n;\beta,\gamma,\alpha) > 0\).
Then \(p(n) > 0\), and so \(G(n;\beta,\gamma,\alpha) =p(n) G_{\Pi}(n;\beta,\gamma,\alpha)\) implying \(G_{\Pi}(n;\beta,\gamma,\alpha) > 0\).
Repeating the steps for \(G_{\Pi}(n;\beta,\gamma,\alpha)\), we obtain the result. 
\end{proof}

%%%%%%%%%%%%%%%%%%%%%%%%%%%%%%%%%%%%%%%%%%%%%%%%%%%%%%%%%%%%%%%%%%%%%%%%%%%%%%%%
\bibliographystyle{IEEEtran}


%%%%%%%%%%%%%%%%%%%%%%%%%%%%%%%%%%%%%%%%%%%%%%%%%%%%%%%%%%%%%%%%%%%%%%%%%%%%%%%%

\end{document}